\documentclass[conference]{IEEEtran}

\usepackage{amsmath,amssymb,amsfonts, amsthm}
\usepackage{xcolor}
\usepackage{array}
\usepackage{booktabs}
\usepackage{microtype}
\usepackage{graphicx}
\usepackage{subfigure}
\usepackage{multirow}
\usepackage{hyperref}

\usepackage[capitalize,noabbrev]{cleveref}

\definecolor{blue}{RGB}{0, 114, 188}
\definecolor{orange}{RGB}{255, 127, 14}
\definecolor{yellow}{RGB}{255, 214, 10}
\definecolor{magenta}{RGB}{255,153,204}
\definecolor{green}{RGB}{44,160,44}
\definecolor{teal}{RGB}{102, 204, 204}
\definecolor{purple}{RGB}{148, 103, 189}
\definecolor{red}{RGB}{214, 39, 40}

\usepackage{tikz}
\usetikzlibrary{shapes.geometric, arrows.meta, positioning, fit, backgrounds, calc, intersections}
\usepackage{atbegshi}
\usepackage{pgfplots}
\usepgfplotslibrary{fillbetween}
\usepgfplotslibrary{groupplots}

\pgfplotsset{compat=newest}

\usepackage{algorithm}
\usepackage{algpseudocode}

\newtheorem{proposition}{Proposition}

\DeclareMathOperator*{\argmin}{\arg\,min}

\newcommand{\calT}[0]{\mathcal{T}}
\newcommand{\calD}[0]{\mathcal{D}}

\newcommand{\calX}[0]{\mathcal{X}}

\title{MuRAL-CPD: Active Learning for Multiresolution Change Point Detection}
\author{%
  \IEEEauthorblockN{%
    Stefano Bertolasi\IEEEauthorrefmark{1}\IEEEauthorrefmark{2}\IEEEauthorrefmark{3},
    Diego Carrera\IEEEauthorrefmark{2},
    Diego Stucchi\IEEEauthorrefmark{2},
    Pasqualina Fragneto\IEEEauthorrefmark{2},
    Luigi Amedeo Bianchi\IEEEauthorrefmark{1}%
  }%
  \IEEEauthorblockA{\IEEEauthorrefmark{1}%
    Università di Trento, Trento, Italy\\
    luigiamedeo.bianchi@unitn.it}
  \IEEEauthorblockA{\IEEEauthorrefmark{2}%
    STMicroelectronics, Agrate Brianza, Italy\\
    \{diego.carrera, diego.stucchi, pasqualina.fragneto\}@st.com}%
    \IEEEauthorblockA{\IEEEauthorrefmark{3}%
    Current address: Politecnico di Milano, Milano, Italy\\
    stefano.bertolasi@polimi.it
    }
}

\begin{document}

\maketitle

\begin{abstract}
Change Point Detection (CPD) is a critical task in time series analysis, aiming to identify moments when the underlying data-generating process shifts. Traditional CPD methods often rely on unsupervised techniques, which lack adaptability to task-specific definitions of change and cannot benefit from user knowledge. To address these limitations, we propose MuRAL-CPD, a novel semi-supervised method that integrates active learning into a multiresolution CPD algorithm. MuRAL-CPD leverages a wavelet-based multiresolution decomposition to detect changes across multiple temporal scales and incorporates user feedback to iteratively optimize key hyperparameters. This interaction enables the model to align its notion of change with that of the user, improving both accuracy and interpretability. Our experimental results on several real-world datasets show the effectiveness of MuRAL-CPD against state-of-the-art methods, particularly in scenarios where minimal supervision is available.
\end{abstract}

\section{Introduction}\label{sec:intro}
Over the past decade, the significant increase in the production and utilization of low-cost sensors has revolutionized industries ranging from wearable devices and autonomous vehicles to medical monitoring systems. These sensors continuously generate vast amounts of temporally ordered data, commonly referred to as time series. Despite this abundance, a considerable portion of the data remains unlabeled, posing a critical challenge for machine learning methods that depend on high-quality annotated datasets.

The growing power of machine learning models has intensified the demand for robust preprocessing tools that can effectively structure and label time series data. The first step to obtain annotated time series is identifying moments when the statistical properties of the data shift. To this aim, Change Point Detection (CPD)~\cite{surveyCPD} has emerged as a key technique for segmenting time series based on such distributional changes. These change points often correspond to meaningful transitions in the underlying process, such as changes in human activity, physiological states, or system behavior.

Supervised CPD approaches, while powerful, are often impractical due to their reliance on extensive annotations and their limited ability to generalize across heterogeneous datasets. Conversely, unsupervised methods are more widely adopted but lack the contextual insight provided by user expertise, often leading to results that are misaligned with real-world needs. To overcome these limitations, the Active Learning (AL) framework has emerged as a promising strategy. By iteratively engaging the user to label the most informative samples, AL enables models to adapt to task-specific requirements while minimizing the amount of labeled data required. This paradigm aligns with the broader concept of human-in-the-loop machine learning~\cite{surveyHITL}, where human feedback is leveraged to guide and improve algorithmic performance.

In this paper, we introduce MuRAL-CPD, a novel AL method for multiresolution CPD. Our approach builds upon a novel unsupervised backbone, MuR-CPD, which leverages multiresolution wavelet decomposition to detect changes across multiple temporal scales. By embedding active learning directly into the unsupervised pipeline, MuRAL-CPD optimizes key hyperparameters through minimal user feedback, ensuring faster convergence and consistent representations across learning phases. This unified design eliminates the need for separate models, reducing computational overhead and enabling a more efficient adaptation to user-defined notions of change.

We evaluate MuRAL-CPD on diverse real-world datasets, demonstrating its competitive performance compared to state-of-the-art methods. Our contributions are threefold:
\begin{itemize}
\item We propose a multiresolution CPD algorithm that combines wavelet-based decomposition with an efficient scoring mechanism to detect changes at varying temporal scales.
\item We integrate active learning into the unsupervised backbone, enabling adaptive refinement of model predictions based on user feedback.
\item We validate our approach through extensive experiments, showing that our method outperforms state-of-the-art competitors.
\end{itemize}

The paper is organized as follows. Section \ref{sec:problem_formulation} formalizes the CPD problem and discusses its challenges and Section \ref{sec:related} reviews related literature. Section \ref{sec:proposed} presents the proposed MuRAL-CPD method, detailing its multiresolution feature extraction, scoring mechanism, and active learning loop. Section \ref{sec:experiments} describes the experimental setup and results, and Section \ref{sec:conclusions} presents the conclusions.

\section{Problem Formulation}\label{sec:problem_formulation}
Time series can be modeled as finite observations of real-world processes evolving over time. To formalize this notion, let $(\Omega, \mathcal{F}, \mathbb{P})$ be a probability space, and let $\{X_i\}_{i \in \mathcal{T}}$, with $\mathcal{T} \subset \mathbb{N}$, denote a discrete-time stochastic process taking values in a measurable space $\mathcal{M}$.

A time series is then defined as a finite realization of this process over a finite index set $\mathcal{T}_0 = \{1, \dots, n\} \subset \mathcal{T}$, for a fixed but unknown outcome $\omega_0 \in \Omega$. Assuming $\mathcal{M} = \mathbb{R}^d$, we write the time series as $x \in \mathbb{R}^{d \times n}$,
with $x[i] := X_i(\omega_0) \in \mathbb{R}^d$.

The goal of CPD is to identify time indices at which the statistical properties of the process change. Specifically, we aim to infer a set of $N>0$, $N$ unknown, change points
\begin{equation*}
    \{\tau_1, \dots, \tau_N\} \subset \mathcal{T}_0
\end{equation*}
such that the distribution of the process changes at each $\tau_j$, i.e.
\begin{equation*}
    X_i \sim
    \begin{cases}
        \phi_j & \text{for } t < \tau_j, \\
        \phi_{j+1} & \text{for } t \geq \tau_j,
    \end{cases}
\end{equation*}
with $\phi_j \neq \phi_{j+1}$ being two unknown distributions.

These changes may reflect shifts in the underlying generative mechanism, such as transitions between different regimes or activities. The CPD task is to detect these transition points solely based on the observed time series $x$, without direct knowledge of the distributions. While traditional CPD methods aim to infer change points purely from the observed time series $x$, in many real-world scenarios the notion of a meaningful change is context-dependent and cannot be fully captured by unsupervised criteria alone.

To address this issue, we consider a semi-supervised variant of the CPD problem, where a user can interactively provide labels--on a limited number of time intervals, in the form of binary annotations--0 for non-CP and 1 for CP. Specifically, let $\mathcal{D}_S$ denote a set of user-provided annotations. These annotations are acquired iteratively via an AL loop. The goal then becomes to infer a set of change points $\{\tau_1, \dots, \tau_N\}$ that not only fit the observed time series $x$, but also align with the information provided in $\mathcal{D}_S$.

\section{Related Works}\label{sec:related}
Unsupervised methods for CPD have been extensively explored due to their generality and independence from labeled data~\cite{review_unsupervisedCPD}. A common class relies on sliding windows to compare adjacent segments. For instance, the \textit{Generalized Likelihood Ratio} (GLR)~\cite{GLR} fits parametric models and tests for distributional shifts via likelihood ratios. Non-parametric alternatives include \textit{RuLSIF}~\cite{RuLSIF}, which estimates the relative Pearson divergence, and \textit{KLIEP}~\cite{KLEIP}, which instead uses the KL divergence, both via direct density ratio estimation. Other notable approaches include Bayesian models~\cite{barry1993bayesian}, which assume a piecewise-constant structure with Markov priors, and symbolic methods like \textit{FLOSS}~\cite{floss} and \textit{CLASP}~\cite{clasp}, based on matrix profiles and motif dissimilarities, respectively.

More recently, \textit{TIRE}~\cite{TIRE} proposed an autoencoder-based model that learns time-invariant representations of fixed-length windows. Change points are then detected by computing dissimilarities between latent embeddings. While effective, TIRE requires careful tuning and suffers from limited interpretability.

Despite their generality, unsupervised methods suffer from several limitations: they often produce spurious detections, lack adaptability to task-specific definitions of change, and are unable to incorporate user feedback. These drawbacks are particularly problematic in applications where only certain types of changes are relevant or where expert knowledge is available but limited.

In recent years, semi-supervised approaches have emerged to address the limitations of unsupervised CPD by incorporating user feedback into the detection process. \textit{AL-CPD}~\cite{AL-CPD} formulates CPD as a binary classification problem, combining an unsupervised detector (TIRE) with a supervised classifier trained on pseudo-labels and refined through active learning. However, the use of two independent models introduces conflicting objectives and makes the approach sensitive to the imbalance of the two classes--CP and non-CP.

To overcome these issues, \textit{ICPD}~\cite{ICPD} was proposed as a unified framework based on one-class classification. It treats CPD as an outlier detection problem and integrates active learning into the training of a one-class SVM. By querying the most uncertain predictions and iteratively retraining the model, ICPD adapts to user preferences while remaining robust to label scarcity. Empirical results show consistent improvements over both its initialization models and AL-CPD across several benchmarks.

Nevertheless, a fundamental limitation of ICPD lies in the architectural separation between the unsupervised model used for initialization and the one-class classifier employed during the active learning phase. This disconnection forces the supervised model to relearn representations already encoded by the detector, resulting in slower convergence and unnecessary redundancy.

In contrast, our method integrates active learning directly within the unsupervised backbone by formulating the learning loop as a hyperparameter optimization problem. Rather than introducing a new model, we refine the parameters governing the score computation itself. This design ensures faster convergence, consistent representation across learning phases, and avoids the computational overhead of maintaining separate models.

\section{Proposed Solution}\label{sec:proposed}

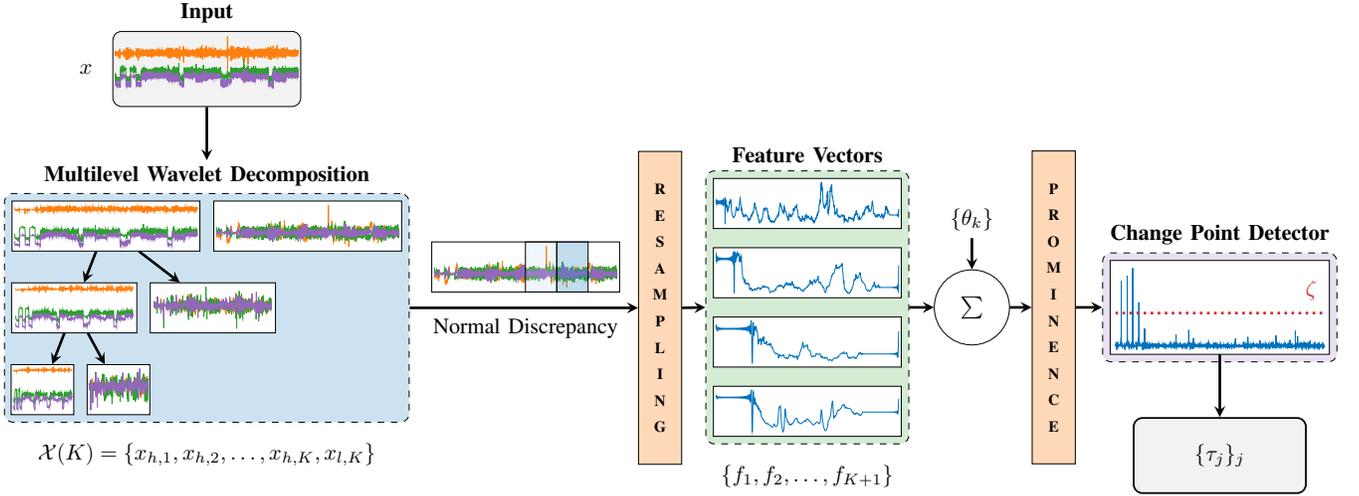
\begin{figure*}[t]
    \centering
    \resizebox{\textwidth}{!}{
    \pgfdeclarelayer{background}
    \pgfsetlayers{background,main}
    \begin{tikzpicture}
        [grid/.style={very thin,gray},
        startstop/.style={draw=black, fill=gray!10, rounded corners, minimum width=2.75cm, minimum height=1.2cm},
        narrow/.style={draw=black, very thin, fill=white, minimum width=3cm, minimum height=0.8cm},
        process/.style={draw=black, very thin, fill=white, font=\small, minimum width=3.5cm, minimum height=1.5cm},
        arrow/.style={very thick,->,>=stealth},
        group/.style={draw=black, dashed, rounded corners, fill=orange!10, fill opacity=0.2},
        ]

        \node[startstop, minimum width = 3cm] (input) at (0,0) {};
        \node[font=\bfseries, above=0.01cm of input.north] {Input};
        \begin{scope}[]
            \begin{axis}[
                at={(input.north)},
                anchor=north,
                width=3.5cm,
                height=1cm,
                axis lines=none,
                scale only axis=true,
                xtick=\empty,
                ytick=\empty,
                clip=true
            ]
            \addplot[orange, very thin] table[x expr=(\coordindex+1), y index=0] {data/samples.txt};
            \addplot[green, very thin] table[x expr=(\coordindex+1), y index=1] {data/samples.txt};
            \addplot[purple, very thin] table[x expr=(\coordindex+1), y index=2] {data/samples.txt};
            \end{axis}
        \end{scope}

        \node[left=0.2cm of input.west] {$x$};

        \node[narrow, below left=1.5cm and -1.4cm of input] (cA1) {};
        \begin{scope}[]
            \begin{axis}[
                at={(cA1.north)},
                anchor=north,
                width=3.5cm,
                height=0.8cm,
                axis lines=none,
                scale only axis=true,
                xtick=\empty,
                ytick=\empty,
                clip=true
            ]
            \addplot[orange, very thin] table[x expr=(\coordindex+1), y index=0] {data/cA1.txt};
            \addplot[green, very thin] table[x expr=(\coordindex+1), y index=1] {data/cA1.txt};
            \addplot[purple, very thin] table[x expr=(\coordindex+1), y index=2] {data/cA1.txt};
            \end{axis}
        \end{scope}
        \node[narrow, below right=1.5cm and -1.4cm of input] (cD1) {};
        \begin{scope}[]
            \begin{axis}[
                at={(cD1.north)},
                anchor=north,
                width=3.5cm,
                height=0.8cm,
                axis lines=none,
                scale only axis=true,
                xtick=\empty,
                ytick=\empty,
                clip=true
            ]
            \addplot[orange, very thin] table[x expr=(\coordindex+1), y index=0] {data/cD1.txt};
            \addplot[green, very thin] table[x expr=(\coordindex+1), y index=1] {data/cD1.txt};
            \addplot[purple, very thin] table[x expr=(\coordindex+1), y index=2] {data/cD1.txt};
            \end{axis}
        \end{scope}

        \node[narrow, below left=0.5cm and -2cm of cA1, minimum width=2cm] (cA2) {};
        \begin{scope}[]
            \begin{axis}[
                at={(cA2.north)},
                anchor=north,
                width=2.25cm,
                height=0.8cm,
                axis lines=none,
                scale only axis=true,
                xtick=\empty,
                ytick=\empty,
                clip=true
            ]
            \addplot[orange, very thin] table[x expr=(\coordindex+1), y index=0] {data/cA2.txt};
            \addplot[green, very thin] table[x expr=(\coordindex+1), y index=1] {data/cA2.txt};
            \addplot[purple, very thin] table[x expr=(\coordindex+1), y index=2] {data/cA2.txt};
            \end{axis}
        \end{scope}
        \node[narrow, below right=0.5cm and -0.8cm of cA1, minimum width=2cm] (cD2) {};
        \begin{scope}[]
            \begin{axis}[
                at={(cD2.north)},
                anchor=north,
                width=2.25cm,
                height=0.8cm,
                axis lines=none,
                scale only axis=true,
                xtick=\empty,
                ytick=\empty,
                clip=true
            ]
            \addplot[orange, very thin] table[x expr=(\coordindex+1), y index=0] {data/cD2.txt};
            \addplot[green, very thin] table[x expr=(\coordindex+1), y index=1] {data/cD2.txt};
            \addplot[purple, very thin] table[x expr=(\coordindex+1), y index=2] {data/cD2.txt};
            \end{axis}
        \end{scope}

        \draw[arrow] (cA1) -- (cA2);
        \draw[arrow] (cA1) -- (cD2);

        \node[narrow, below left=0.5cm and -1cm of cA2, minimum width=1cm] (cAK) {};
        \begin{scope}[]
            \begin{axis}[
                at={(cAK.north)},
                anchor=north,
                width=1.125cm,
                height=0.8cm,
                axis lines=none,
                scale only axis=true,
                xtick=\empty,
                ytick=\empty,
                clip=true
            ]
            \addplot[orange, very thin] table[x expr=(\coordindex+1), y index=0] {data/cAK.txt};
            \addplot[green, very thin] table[x expr=(\coordindex+1), y index=1] {data/cAK.txt};
            \addplot[purple, very thin] table[x expr=(\coordindex+1), y index=2] {data/cAK.txt};
            \end{axis}
        \end{scope}
        \node[narrow, below right=0.5cm and -0.8cm of cA2, minimum width=1cm] (cDK) {};
        \begin{scope}[]
            \begin{axis}[
                at={(cDK.north)},
                anchor=north,
                width=1.125cm,
                height=0.8cm,
                axis lines=none,
                scale only axis=true,
                xtick=\empty,
                ytick=\empty,
                clip=true
            ]
            \addplot[orange, very thin] table[x expr=(\coordindex+1), y index=0] {data/cDK.txt};
            \addplot[green, very thin] table[x expr=(\coordindex+1), y index=1] {data/cDK.txt};
            \addplot[purple, very thin] table[x expr=(\coordindex+1), y index=2] {data/cDK.txt};
            \end{axis}
        \end{scope}

        \draw[arrow] (cA2) -- (cAK);
        \draw[arrow] (cA2) -- (cDK);

        \begin{pgfonlayer}{background}
            \node[group, fit=(cA1)(cD1)(cDK), label=above:{\textbf{Multilevel Wavelet Decomposition}}, fill=blue]  (MDWD_group) {};
        \end{pgfonlayer}
        
        \node at ($(MDWD_group.north) + (0,0.4cm)$) (MDWD_group_anchor) {};
        \draw[arrow] (input.south) -- (MDWD_group_anchor);

        \node[below=0.2cm of MDWD_group.south] {$\calX(K) = \{x_{h,1}, x_{h,2}, \dots, x_{h,K}, x_{l,K}\}$};

        \node[draw=black, fill=orange!30, minimum width=0.1cm, minimum height=5cm, inner sep=0.5pt] (resample) at ($(MDWD_group.east) + (4, 0)$) {
        \begin{tabular}{c}
            {\scriptsize \textbf{R}} \\
            {\scriptsize \textbf{E}} \\
            {\scriptsize \textbf{S}} \\
            {\scriptsize \textbf{A}} \\
            {\scriptsize \textbf{M}} \\
            {\scriptsize \textbf{P}} \\
            {\scriptsize \textbf{L}} \\
            {\scriptsize \textbf{I}} \\
            {\scriptsize \textbf{N}} \\
            {\scriptsize \textbf{G}}
        \end{tabular}
    };
        \draw[arrow] (MDWD_group.east) -- (resample.west);
        
        \node[narrow] (ns1) at ($(resample.east) + (2, 1.668)$)  {};
        \begin{scope}[]
            \begin{axis}[
                at={(ns1.north)},
                anchor=north,
                width=3.5cm,
                height=0.8cm,
                axis lines=none,
                scale only axis=true,
                xtick=\empty,
                ytick=\empty,
                clip=true
            ]
            \addplot[blue, very thin] table[x expr=(\coordindex+1), y index=0] {data/f0.txt};
            \end{axis}
        \end{scope}
        
        \node[narrow, below=0.3cm of ns1] (ns2) {};
        \begin{scope}[]
            \begin{axis}[
                at={(ns2.north)},
                anchor=north,
                width=3.5cm,
                height=0.8cm,
                axis lines=none,
                scale only axis=true,
                xtick=\empty,
                ytick=\empty,
                clip=true
            ]
            \addplot[blue, very thin] table[x expr=(\coordindex+1), y index=0] {data/f1.txt};
            \end{axis}
        \end{scope}

        \node[narrow, below=0.3cm of ns2] (nsK) {};
        \begin{scope}[]
            \begin{axis}[
                at={(nsK.north)},
                anchor=north,
                width=3.5cm,
                height=0.8cm,
                axis lines=none,
                scale only axis=true,
                xtick=\empty,
                ytick=\empty,
                clip=true
            ]
            \addplot[blue, very thin] table[x expr=(\coordindex+1), y index=0] {data/f2.txt};
            \end{axis}
        \end{scope}

        \node[narrow, below=0.3cm of nsK] (nsK1) {};
        \begin{scope}[]
            \begin{axis}[
                at={(nsK1.north)},
                anchor=north,
                width=3.5cm,
                height=0.8cm,
                axis lines=none,
                scale only axis=true,
                xtick=\empty,
                ytick=\empty,
                clip=true
            ]
            \addplot[blue, very thin] table[x expr=(\coordindex+1), y index=0] {data/f3.txt};
            \end{axis}
        \end{scope}

        \begin{pgfonlayer}{background}
            \node[group, fit=(ns1)(ns2)(nsK)(nsK1), label=above:{\textbf{Feature Vectors}}, fill=green]  (feature_group) {};
        \end{pgfonlayer}

        \node[below=0.2cm of feature_group.south] {$\{f_1, f_2, \dots, f_{K+1}\}$};

        \node[narrow, above right=0.25cm and 0.35cm of MDWD_group.east] (windowing) {};
        \begin{scope}[]
            \begin{axis}[
                at={(windowing.north)},
                anchor=north,
                width=3.5cm,
                height=0.8cm,
                axis lines=none,
                scale only axis=true,
                xtick=\empty,
                ytick=\empty,
                clip=true
            ]
            \addplot[orange, very thin] table[x expr=(\coordindex+1), y index=0] {data/cD1.txt};
            \addplot[green, very thin] table[x expr=(\coordindex+1), y index=1] {data/cD1.txt};
            \addplot[purple, very thin] table[x expr=(\coordindex+1), y index=2] {data/cD1.txt};
            \end{axis}
        \end{scope}

        \node at ($(windowing.west) + (1.5, 0.4)$) (west_angle) {};
        \node at ($(windowing.west) + (2, -0.4)$) (east_angle) {};
        \node at ($(windowing.west) + (2.5, 0.4)$) (west_angle2) {};
        \draw[fill=blue!30, fill opacity=0.2, thin] (west_angle) rectangle (east_angle);
        \draw[fill=blue!80, fill opacity=0.3, thin] (west_angle2) rectangle (east_angle);
        \node[below=0.3cm of windowing.south] {Normal Discrepancy};

        \draw[arrow] (resample.east) -- (feature_group.west);

        \node[draw=black, circle, minimum size=1.2cm, fill=white] (fusion) at ($(feature_group.east) + (1, 0)$) {$\sum$};

        \node[above=0.5cm of fusion.north] (theta) {$\{\theta_k\}$};
        \draw[arrow] (theta.south) -- (fusion.north);

        \draw[arrow] (feature_group.east) -- (fusion.west);

        \node[draw=black, fill=orange!30, minimum width=0.1cm, minimum height=5cm, inner sep=0.5pt] (prominence) at ($(fusion.east) + (0.7, 0)$) {\begin{tabular}{c}
            {\scriptsize \textbf{P}} \\
            {\scriptsize \textbf{R}} \\
            {\scriptsize \textbf{O}} \\
            {\scriptsize \textbf{M}} \\
            {\scriptsize \textbf{I}} \\
            {\scriptsize \textbf{N}} \\
            {\scriptsize \textbf{E}} \\
            {\scriptsize \textbf{N}} \\
            {\scriptsize \textbf{C}} \\
            {\scriptsize \textbf{E}}
        \end{tabular}
    };

        \draw[arrow] (fusion.east) -- (prominence.west);

        \node[process] (cpd) at ($(prominence.east) + (2.3, 0)$) {};
        
        \begin{scope}[]
            \begin{axis}[
                at={(cpd.north)},
                anchor=north,
                width=4cm,
                height=1.5cm,
                axis lines=none,
                scale only axis=true,
                xtick=\empty,
                ytick=\empty,
                clip=true
            ]
            \addplot[blue, very thin] table[x expr=(\coordindex+1), y index=0] {data/prominence.txt};
            \addplot[red, very thick, dotted] coordinates {(0,0.1) (2000,0.1)};
            \end{axis}
        \end{scope}

        \begin{pgfonlayer}{background}
            \node[group, fit=(cpd), label=above:{\textbf{Change Point Detector}}, fill=purple]  (cpd_group) {};
        \end{pgfonlayer}
        \draw[arrow] (prominence.east) -- (cpd_group.west);

        \node[startstop, below=1cm of cpd, thick] (output) {$\{\tau_j\}_j$};
        \draw[arrow] (cpd.south) -- (output.north);

        \node[above left=0cm and 0.2cm of cpd_group.east] {\textcolor{red}{$\zeta$}};        
    \end{tikzpicture} 
    }
    \caption{Overview of the MuR-CPD algorithm. The input time series undergoes a Multilevel Wavelet Decomposition to extract multiresolution subbands. We compute discrepancy scores on each subband using a sliding window approach, followed by resampling to obtain feature vectors aligned with the original time index. Then we aggregate the features by a linear combination and apply the prominence function $\pi$ to highlight peaks, and change points are detected by applying a tunable decision threshold.}
    \label{fig:overview}
\end{figure*}   

In this section, we present MuRAL-CPD, a novel AL method for CPD. MuRAL-CPD integrates MuR-CPD, a novel unsupervised multiresolution CPD algorithm, with an AL strategy designed to optimize the parameters of MuR-CPD through minimal user feedback.

As illustrated in Figure~\ref{fig:overview}, MuR-CPD operates in several stages.
First, we produce a hierarchical decomposition of the time series by applying a multiresolution wavelet transform to each of its channels. The resulting sub-bands capture distinct temporal and frequency patterns that emerge at different scales. This hierarchical decomposition enables the algorithm to locate changes across multiple resolutions, thus detecting both abrupt and subtle CP.

Then, for each extracted sub-band we compute a discrepancy score, which quantifies the likelihood of a change occurring at each time instant. This is done by evaluating the discrepancy between adjacent segments using the sliding window approach proposed in~\cite{normal_discrepancy_score}. To reconcile the decimation introduced by the wavelet transform and maintain temporal alignment among these discrepancy scores, we resample them to match the original time interval. We refer to the resampled outputs as \emph{features}, each indicating the likelihood--assessed at a different scale--of a change occurring at each time instant.

Finally, we aggregate these features into a final score through a linear combination, we highlight its peaks through a prominence function, and we detect CP by thresholding these peaks.
 
While MuR-CPD is a fast and effective unsupervised method, MuRAL-CPD further improves the algorithm accuracy by incorporating active learning. In MuRAL-CPD, user-provided feedback is used to fine-tune critical hyperparameters of the unsupervised pipeline, thus enhancing the robustness and precision of the algorithm.

\subsection{Feature Extraction}\label{subsec:feature_extraction}
In this section, we describe the process by which we extract informative features from the input time series. This procedure progressively transforms the raw signal into a multiresolution representation, and subsequently into a form that quantifies local statistical changes over time.

To extract a multiresolution representation of the signal, we consider the Multilevel Discrete Wavelet Decomposition (MDWD)~\cite{wavelet}. MDWD is a wavelet-based discrete signal analysis technique that extracts time-frequency features by iteratively decomposing the input series into lower- and higher-frequency sub-bands. For each resolution level, a pair of low-pass and high-pass filters $l$ and $h$ are applied to isolate the approximation and detail components of the signal, respectively.

Formally, let $x$ be the original time series. We denote by $x_{l,k}$ and $x_{h,k}$ the approximation and detail sub-bands obtained at level $k$. These are computed by convolution of the filters $l$ and $h$ against the previous detail components $x_{l,k-1}$:
\begin{align}\label{eq:conv}
x_{l,k} = l \circledast x_{l,k-1} \\
x_{h,k} = h \circledast x_{l,k-1}
\end{align}
where the output of each convolution is downsampled by a factor of $2$ and $x_{l,0} := x$ is the original signal. The operation is repeated recursively on $x_{l,k}$ to obtain the next level of decomposition. The sub-bands set $\calX(K) ~=~ \{x_{h,1}, x_{h,2}, \dots, x_{h,K}, x_{l,K}\}$ is the MDWD of the time series $x$ into $K$ levels.

To detect the local distributional shifts, we adopt a sliding window approach where we compute a discrepancy measure between two adjacent segments. 
In particular, let $\bar{x} \in \mathcal{X}(K)$ be a sub-bands obtained from the MDWD, and let $w > 0$ be the window size. At each time instant $i$, we define the $i$-left, $i$-right and $i$-centered windows as:
\begin{equation}
    \begin{aligned}
        &W_i^L(\bar{x}) := \bar{x}[i-w+1 : i], \\
        &W_i^R(\bar{x}) := \bar{x}[i+1 : i+w], \\
        &W_i(\bar{x})   := \bar{x}[i-w+1 : i+w].
    \end{aligned}
\end{equation}
Then, we measure the \emph{Normal Discrepancy}~\cite{normal_discrepancy_score} at the instant $i$ as:
\begin{equation}
    \Tilde{f}_k[i] := w \ln \left( \frac{\det \Sigma}{\sqrt{\det \Sigma_L \cdot \det \Sigma_R}} \right),
\end{equation}
where $\Sigma$, $\Sigma_L$, and $\Sigma_R$ denote the sample covariance matrices of $W_i(x_{l,k})$, $W_i^L(x_{l,k})$, and $W_i^R(x_{l,k})$, respectively. This measure captures local discrepancies in both the mean and variance of the signal and remains valid even in the presence of temporal dependencies, extending beyond the i.i.d.~setting.

Importantly, since MDWD produces a set of sub-bands at different time resolutions, using the same window size $w$ across levels corresponds to evaluating statistical changes over increasingly larger time scales. As a result, the sliding window covers a larger portion of the original signal at coarser levels, thus allowing our method to detect CP that occur at different temporal scales.

Finally, to realign all discrepancy vectors in a common temporal domain, we apply a resampling step that interpolates each $\Tilde{f}_k$ back to the original time interval of $x$, resulting in a feature vector $f_k$. This ensures that all multiscale profiles are temporally aligned and can be consistently combined in the subsequent stages of the model.
The set of features $\{f_k\}_k, k=1, \dots K+1$,  provide a multi-resolution representation of the time series in the frequency domain.

\subsection{Change Point Detection}\label{subsec:change_detection}
To compute an aggregated score vector indicating the likelihood of each sample being a CP, we linearly combine the sub-band feature representations $\{f_k\}_{k=1}^{K+1}$ and apply a prominence transformation $\pi$ to emphasize significant peaks:
\begin{equation}
    s[i] = \pi \left(\sum_{k=1}^{K+1} \theta_k f_k[i]\right),
\end{equation}
where $\{\theta_k\}_{k=1}^{K+1}$ are tunable parameters controlling the contribution of each resolution level.

The prominence transformation $\pi$ assigns a prominence value to each local maximum, quantifying how much it stands out relative to its surroundings. Formally, the prominence of a peak at index $i$ is defined as:
\begin{equation}
    \pi(f)[i] = f[i] - \max\left( \min_{j \in [i_-, i]}f[j], \min_{j \in [i, i_+]}f[j] \right),
\end{equation}
where $(i_-, i_+)$ is the largest interval containing $i$ such that $f[i] > f[j]$ for all $j \in (i_-, i_+)$ and
\begin{equation}
    f = \sum_{k=1}^{K+1} \theta_k f_k.
\end{equation}
For all non-peak points, the prominence is defined to be $0$.

Finally, to perform change point detection, we introduce a learnable threshold parameter $\zeta > 0$. Time instants $\tau \in \calT_0$ whose score $s[\tau]$ exceeds the threshold $\zeta$ are classified as change points.

\subsection{AL for Hyperparameters Tuning}
The algorithm described so far provides a solid foundation for unsupervised CPD. To further enhance its performance and adaptability, we integrate an AL phase that enables the user to iteratively correct and personalize model predictions.
This is achieved by integrating user feedback as annotated data to optimize a set of hyperparameters that influence the resulting score vector--namely, the feature weights $\{\theta_k\}_k$ and the detection threshold $\zeta$--thus influencing the detection outcome.

The weights $\{\theta_k\}_k$ control the importance of each resolution level within the embedding space. Increasing the weight of a specific feature relative to the others encourages the model to focus more on the corresponding sub-band, thereby refining its sensitivity to structural changes at that scale. The decision threshold $\zeta$ determines the minimum score required for a point to be classified as a change point, effectively governing the trade-off between precision and recall.
By incorporating user feedback into the optimization loop, the method adapts to task-specific definitions of change, yielding a more accurate detection process. The full procedure is presented below and detailed in Algorithm~\ref{alg:active_learning}.

The active learning loop starts by initializing the supervised and unsupervised domains (line~\ref{alg:initialization}). At each iteration (line~\ref{alg:loopstart}), the algorithm identifies the most uncertain candidates by selecting the scores closest to the current threshold $\zeta$ from above and below (lines~\ref{alg:iplus}-\ref{alg:iminus}):
\begin{equation}
i^+ = \argmin_{i \in \calD_U, \, s[i] \geq \zeta} |s[i] - \zeta|, \quad
i^- = \argmin_{i \in \calD_U, \, s[i] < \zeta} |s[i] - \zeta|.
\end{equation}
For each selected index (line~\ref{alg:forallidx}), we define a local window $W_i = [i - p, i + p]$ and ask the user to annotate whether change points are present within it (lines~\ref{alg:window}-\ref{alg:userfb}). The labeled samples are then added to the supervised set $\mathcal{D}_S$, while the queried region is removed from $\mathcal{D}_U$ (lines~\ref{alg:supervised}-\ref{alg:unsupervised}).

The updated supervised set is used to optimize the parameters $(\{\theta_k\}, \zeta)$ (line~\ref{alg:optimization}), minimizing the loss:
\begin{equation}
    \mathcal{L}(\{\theta_k\}, \zeta) = 1 - F_1(\mathcal{D}_S),
\end{equation}
where the $F_1$-score is computed on the user-labeled points. The new parameters are then used to recompute the score sequence (line~\ref{alg:newscore}), and the process is repeated until the supervision budget $b$ is exhausted.

This iterative loop (lines~\ref{alg:loopstart}-\ref{alg:loopend}) progressively adapts the model to the user's notion of relevant change, enabling a personalized and data-efficient refinement of the initial unsupervised predictions.

\begin{algorithm}[t]
\caption{Active Learning Loop}
\begin{algorithmic}[1]
\Require Time series $x$, initial scores $s$, initial parameters $(\{\theta_k\}, \zeta)$, budget $b$, tolerance $\eta$
\State Initialize supervised domain $\calD_S \gets \emptyset$ and unsupervised domain $\calD_U=\calT_0$ \label{alg:initialization}
\For{$m = 1$ to $b$} \label{alg:loopstart}
    \State Compute $i^+ \gets \argmin_{i \in \calD_U,\, s[i] \geq \tau} |s[i] - \tau|$ \label{alg:iplus}
    \State Compute $i^- \gets \argmin_{i \in \calD_U,\, s[i] < \tau} |s[i]
    - \tau|$ \label{alg:iminus}
    \ForAll{$i \in \{i^+, i^-\}$} \label{alg:forallidx}
        \State $W_i = [i - p, i + p]$ \label{alg:window}
        \State Query the user for true change points $\{\tau_j\}_j \!\subset\! W_i$ \label{alg:userfb}
        \State $\calD_S \gets \calD_S \cup \{\tau_j\}_j$ \label{alg:supervised}
        \State $\calD_U \gets \calD_U \setminus W_i$ \label{alg:unsupervised}
    \EndFor
    \State ${(\{\theta_k\}, \zeta) = \argmin{\mathcal{L}(\{\theta_k\}, \zeta)}}$ \label{alg:optimization}
    \State Obtain the new score $s = \pi \left(\sum_{k=1}^{K+1} \theta_k f_k\right)$ \label{alg:newscore}
\EndFor \label{alg:loopend}
\State \Return Final parameters $(\{\theta_k\}, \zeta)$ \label{alg:final}
\end{algorithmic}
\label{alg:active_learning}
\end{algorithm}
\subsection{Hyperparameter Initialization}
\label{sec:hyper_init}
\begin{figure}[t]
  \centering
  \begin{tikzpicture}
    \begin{axis}[
      width=\columnwidth,
      height=0.6\columnwidth,
      title={},
      xlabel={$t$},
      ylabel={$\gamma(t)$},
      xmin=0, xmax=1,
      ymin=0, ymax=2.5,
      xtick distance=0.2,
      ytick distance=0.5,
    ]
      \addplot[
        blue,
        very thick,
        mark=none
      ] coordinates {
        (0.00, 2.50)
        (0.02, 2.00)
        (0.04, 1.70)
        (0.06, 1.40)
        (0.08, 1.10)
        (0.10, 0.90)
        (0.12, 0.85) 
        (0.14, 0.80)
        (0.16, 0.78)
        (0.18, 0.75)
        (0.20, 0.70)  
        (0.24, 0.68)
        (0.28, 0.65)
        (0.32, 0.60)
        (0.36, 0.58)
        (0.40, 0.55)
        (0.50, 0.50)
        (0.60, 0.45)
        (0.70, 0.40)
        (0.80, 0.30)
        (0.90, 0.20)
        (1.00, 0.10)
      };

      \addplot[
        only marks,
        mark=*,
        mark size=3pt,
        draw=black,
        fill=yellow
      ] coordinates { (0.10, 0.90) };

    \draw[dashed,gray] (axis cs:0.10, 0.90) -- (axis cs:0.10,0);
    \draw[dashed,gray] (axis cs:0.10, 0.90) -- (axis cs:0,0.90);

    \node[label=below:$t^*$] 
        at (axis cs:0.10,0) {};
    \node[label=left:$\zeta_0$] 
        at (axis cs:0,0.90) {};
    
    \end{axis}
  \end{tikzpicture}
  \caption{The sorted score function $\gamma(t)$ plotted against the normalized index $t\in[0,1]$. The piecewise linear curve exhibits a steep initial drop followed by a gradual tail. The elbow point at $t\approx0.10$ (yellow circle) corresponds to the maximum curvature of $\gamma(t)$. Dashed lines project this point horizontally to $\zeta_0\approx0.90$ on the $\gamma$‐axis and vertically to $t$ on the $t$‐axis; $\zeta_0$ is selected as the initial decision threshold.}
  \label{fig:curvature}
\end{figure}
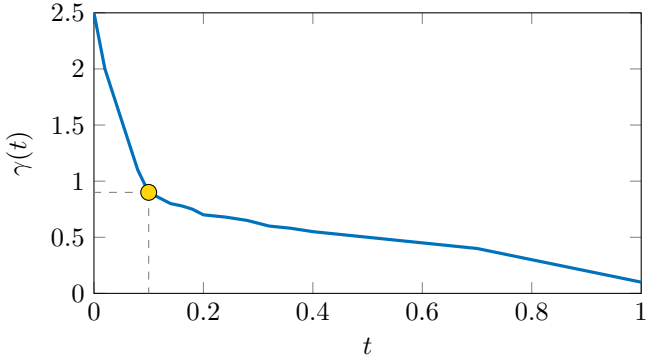
The proposed method involves the optimization of both the feature weights and the decision threshold, which jointly determine the model's predictions. While the weights are uniformly initialized, the initialization of the threshold plays a far more crucial role, strongly influencing the early behavior of the AL-module and affecting both the convergence speed of the entire algorithm and the quality of the queries themselves.

For this reason, a more sophisticated mathematical strategy has been employed.
Since we expect change points to be characterized by significantly larger scores than other points, we initialize the threshold $\zeta$ to maximize the difference between detected and undetected points. Inspired by clustering problems, we cast the threshold initialization as the problem of identifying the most abrupt change in the sorted score profile which ideally represent the best separation between change points and non change points.

First, we compute the scores $s$ associated with the initialized weights, and we normalize them to $[0,1]$.
Then, we sort the scores in decreasing order, resulting in a sequence $i_m, m=1,\dots,n$ such that $s[i_{m'}] > s[i_{m''}]$ for $m'<m''$.
By mapping each time index $i_m \mapsto \frac{m-1}{n-1}$, we construct the discrete curve $\gamma$ in the 2D space $[0, 1]^2$ as the piecewise linear parametrized curve:
\begin{equation}
    \gamma: t \mapsto s(t),
\end{equation}
where $\gamma(\frac{m-1}{n-1}) = s[i_m]$ for every $m=1,\dots,m$.

This curve exhibits a steep initial descent followed by a flattening region, as illustrated in Figure~\ref{fig:curvature}. To determine an initial threshold that separates meaningful scores from negligible ones we identify the point where the curve presents the sharpest change in slope, commonly known as \emph{elbow}--corresponding to the maximum of its \textit{curvature}. Intuitively, the curvature describes for any part of the curve how much the curve direction changes over a small distance traveled. Formally, consider the above regular planar curve $\gamma(t) = (t, s(t)) \in [0,1]^2$. The curvature of a point $t \in [0,1]$ can be computed as:
\begin{equation}
    \kappa(t) := \frac{|s^{\prime \prime}(t)|}{\left (1 + s^\prime(t)^2 \right)^{3/2}}
\end{equation}
where $s^{\prime}(t)$ and $s^{\prime \prime}(t)$ denote the first an the second derivatives of the curve. In the discrete setting, these derivatives are numerically approximated using central differences.

The maximum curvature point $t^* = \arg\max_t\kappa(t)$ corresponds to the hardest inflection in the curve--where the concavity changes most rapidly. This inflection corresponds to a shift in the nature of the points in its neighborhood. For this reason, we associated the score of this point with the value of our initial threshold $\zeta_0$.

This approach provides a non-parametric and sophisticated initialization method based solely on the geometry of the score function, which significantly improves model's performance in the early phase of the AL loop. However, despite its impact, Proposition~\ref{prop:1} shows that the final detection result is invariant under a joint rescaling of both the threshold and the feature weights. In particular, we prove that for any scaling of the threshold, one can construct a corresponding rescaling of the weights that yields the same set of detected change points. 

\begin{proposition}\label{prop:1}
Let $\{\theta_k\}_{k=1}^{K+1} \in \mathbb{R}^{K+1}_{\geq 0}$ be a set of non-negative weights and let $\zeta > 0$ be a decision threshold. Define the aggregated score as
\begin{equation*}
    s[i] = \pi \left(\sum_{k=1}^{K+1} \theta_k \, f_k[i]\right),
\end{equation*}
where $\pi : \mathbb{R}^n \to \mathbb{R}^n$ denote the prominence transformation. The set of detected CP is defined as
\begin{equation*}
    \mathcal{C} := \left\{ i \in \mathcal{T}_0 \,\middle|\, s[i] \geq \zeta \right\}.
\end{equation*}

Then, for every $\Hat{\zeta} > 0$, there exists a set of weights $\{\Hat{\theta}_k\}_{k=1}^{K+1}$ such that the resulting set of detected CP remains unchanged, i.e.,
\begin{equation*}
    \Hat{\mathcal{C}} := \left\{ i \in \mathcal{T}_0 \,\middle|\, \pi\left( \sum_{k=1}^{K+1} \Hat{\theta}_k \, f_k\right)[i]  \geq \Hat{\zeta} \right\} = \mathcal{C}.
\end{equation*}
\end{proposition}

\begin{proof}
From the definition of the prominence transformation (see \ref{subsec:change_detection}), it follows that $\pi$ is homogeneous, i.e. $\pi(\lambda s)=\lambda\pi(s)$. Let us define $\lambda = \Hat{\zeta}/\zeta$, and a the new set of weights ${\Hat{\theta}_k}$ as
\begin{equation}
    \Hat{\theta}_k = \lambda \theta_k, \quad k=1,\dots,K+1.
\end{equation}
If we consider the new aggregated feature vector
\begin{equation*}
    \hat{f}[i] := \sum_{k=1}^{K+1} \Hat{\theta}_k \, f_k[i] = \lambda f[i],
\end{equation*}
from the homogeneity of the prominence transformation, it follows that
\begin{equation*}
    \pi(\Hat{f})[i] \geq \Hat{\zeta} \iff \lambda \pi(f)[i] \geq \lambda \zeta \iff \pi(f)[i] \geq \zeta.
\end{equation*}
Thus, $\Hat{\mathcal{C}} = \mathcal{C}$.
\end{proof}

This property ensures that the final outcome of MuR-CPD is not fundamentally dependent on the specific initial value of $\zeta$, although setting a proper value of $\zeta$ reduces the number of queries required to reach the optimal performance.

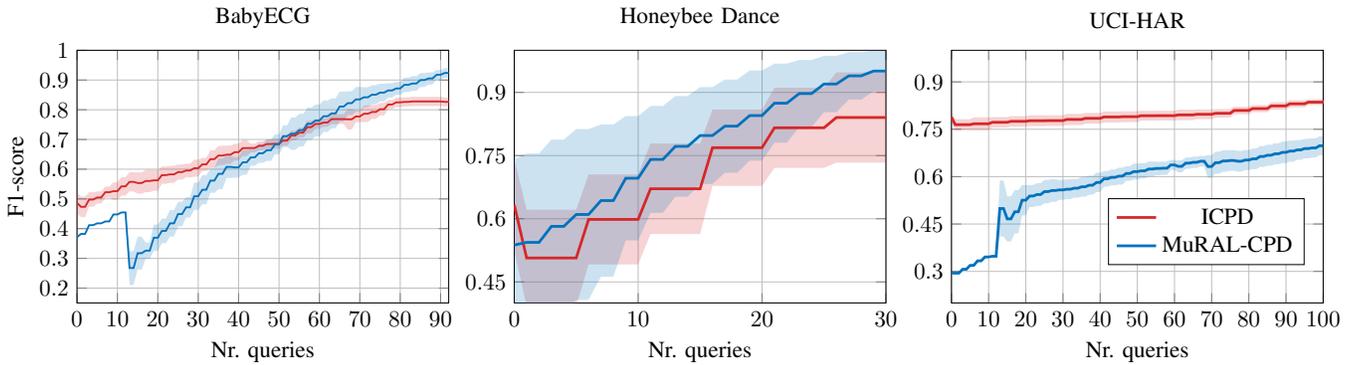
\begin{figure*}[t]
  \centering
  \resizebox{\textwidth}{!}{
  \begin{tikzpicture}
    \begin{groupplot}[
      group style={
        group size=3 by 1,
      },
      width=0.4\textwidth,
      height=0.3\textwidth,
      xlabel={Nr.\ queries},
      ylabel={F1‐score},
      grid=both,
      grid style={line width=.1pt, draw=gray!20},
      major grid style={line width=.2pt, draw=gray!50},
    ]

      \nextgroupplot[title={BabyECG},
        xmin=0, xmax=92,
        ymin=0.15, ymax=1.0,
        xtick distance=10,
        ytick distance=0.1,
      ]
        \addplot[name path=upper, draw=none, forget plot]
          table[
            col sep=space,
            header=true,
            x=query,
            y expr=\thisrow{mean_f1_adj} + \thisrow{std_f1_adj}
          ] {data/results/babyecg_icpd.txt};
        \addplot[name path=lower, draw=none, forget plot]
          table[
            col sep=space,
            header=true,
            x=query,
            y expr=\thisrow{mean_f1_adj} - \thisrow{std_f1_adj}
          ] {data/results/babyecg_icpd.txt};
        \addplot[fill=red, fill opacity=0.2, forget plot] 
          fill between[of=upper and lower];
        \addplot[
          red,
          thick,
        ]
          table[
            col sep=space,
            header=true,
            x=query,
            y=mean_f1_adj
          ] {data/results/babyecg_icpd.txt};
          
        \addplot[name path=upper, draw=none, forget plot]
          table[
            col sep=space,
            header=true,
            x=query,
            y expr=\thisrow{mean_f1} + \thisrow{std_f1}
          ] {data/results/babyecg_muralcpd.txt};
        \addplot[name path=lower, draw=none, forget plot]
          table[
            col sep=space,
            header=true,
            x=query,
            y expr=\thisrow{mean_f1} - \thisrow{std_f1}
          ] {data/results/babyecg_muralcpd.txt};
        \addplot[fill=blue, fill opacity=0.2, forget plot] 
          fill between[of=upper and lower];
        \addplot[
          blue,
          thick,
        ]
          table[
            col sep=space,
            header=true,
            x=query,
            y=mean_f1
          ] {data/results/babyecg_muralcpd.txt};

      \nextgroupplot[title={Honeybee Dance},
        ylabel={},
        xmin=0, xmax=30,
        ymin=0.4, ymax=1.0,
        xtick distance=10,
        ytick distance=0.15,]
        \addplot[name path=upper, draw=none, forget plot]
          table[
            col sep=space,
            header=true,
            x=query,
            y expr=\thisrow{mean_f1_adj} + \thisrow{std_f1_adj}
          ] {data/results/honeybee_icpd.txt};
        \addplot[name path=lower, draw=none, forget plot]
          table[
            col sep=space,
            header=true,
            x=query,
            y expr=\thisrow{mean_f1_adj} - \thisrow{std_f1_adj}
          ] {data/results/honeybee_icpd.txt};
        \addplot[fill=red, fill opacity=0.2, forget plot] 
          fill between[of=upper and lower];
        \addplot[
          red,
          very thick,
        ]
          table[
            col sep=space,
            header=true,
            x=query,
            y=mean_f1_adj
          ] {data/results/honeybee_icpd.txt};
        
        \addplot[name path=upper, draw=none, forget plot]
          table[
            col sep=space,
            header=true,
            x=query,
            y expr=\thisrow{mean_f1} + \thisrow{std_f1}
          ] {data/results/honeybee_muralcpd.txt};
        \addplot[name path=lower, draw=none, forget plot]
          table[
            col sep=space,
            header=true,
            x=query,
            y expr=\thisrow{mean_f1} - \thisrow{std_f1}
          ] {data/results/honeybee_muralcpd.txt};
        \addplot[fill=blue, fill opacity=0.2, forget plot] 
          fill between[of=upper and lower];
        \addplot[
          blue,
          very thick,
        ]
          table[
            col sep=space,
            header=true,
            x=query,
            y=mean_f1
          ] {data/results/honeybee_muralcpd.txt};
    
      \nextgroupplot[title={UCI-HAR},
        ylabel={},
        xmin=0, xmax=100,
        ymin=0.2, ymax=1.0,
        xtick distance=10,
        ytick distance=0.15,
        legend style={
            at={(0.95,0.15)},
            anchor= south east,
       },]
        \addplot[name path=upper, draw=none, forget plot]
          table[
            col sep=space,
            header=true,
            x=query,
            y expr=\thisrow{mean_f1_adj} + \thisrow{std_f1_adj}
          ] {data/results/ucihar_icpd.txt};
        \addplot[name path=lower, draw=none, forget plot]
          table[
            col sep=space,
            header=true,
            x=query,
            y expr=\thisrow{mean_f1_adj} - \thisrow{std_f1_adj}
          ] {data/results/ucihar_icpd.txt};
        \addplot[fill=red, fill opacity=0.2, forget plot] 
          fill between[of=upper and lower];
        \addplot[
          red,
          very thick,
        ]
          table[
            col sep=space,
            header=true,
            x=query,
            y=mean_f1_adj
          ] {data/results/ucihar_icpd.txt};
        
        \addplot[name path=upper, draw=none, forget plot]
          table[
            col sep=space,
            header=true,
            x=query,
            y expr=\thisrow{mean_f1} + \thisrow{std_f1}
          ] {data/results/ucihar_muralcpd.txt};
        \addplot[name path=lower, draw=none, forget plot]
          table[
            col sep=space,
            header=true,
            x=query,
            y expr=\thisrow{mean_f1} - \thisrow{std_f1}
          ] {data/results/ucihar_muralcpd.txt};
        \addplot[fill=blue, fill opacity=0.2, forget plot] 
          fill between[of=upper and lower];
        \addplot[
          blue,
          very thick,
        ]
          table[
            col sep=space,
            header=true,
            x=query,
            y=mean_f1
          ] {data/results/ucihar_muralcpd.txt};

        \addlegendimage{red, thick}\addlegendentry{ICPD}
        \addlegendimage{blue, thick}\addlegendentry{MuRAL‐CPD}

    \end{groupplot}
  \end{tikzpicture}
   }
  \caption{Comparison between F1-scores achieved by ICPD and MuRAL-CPD on benchmark datasets across 10 repetitions. Solid lines represent the mean values across 10 repetitions, while shaded regions indicate the standard deviation across these repetitions.}
  \label{fig:results}
\end{figure*}

\section{Experiments}\label{sec:experiments}
In this section, we present the experimental setup and empirical results that validate the effectiveness of our proposed method. We evaluate MuRAL-CPD on four real-world time series datasets across different domains, comparing its performance against a state-of-the-art semi-supervised change point detection method. We describe the configuration of the feature extraction pipeline and of the active learning protocol, and the evaluation metrics used. Finally, we provide both quantitative and qualitative analyses to highlight the advantages of our approach in terms of accuracy, convergence speed, and sample efficiency.

The Python code can be downloaded at \href{https://github.com/stefanobertolasi/mural_cpd.git}{github.com/stefanobertolasi/mural\_cpd.git}.

\begin{table}[t]
\centering
\caption{Hyperparameters used for each dataset. $K$ is the number of wavelet decomposition levels, $w$ is the sliding window size used in the discrepancy score computation, and $\eta$ is the tolerance window used for active learning queries.}
\label{tab:dataset_parameters}
\vspace{0.2cm}
\resizebox{\columnwidth}{!}{%
\begin{tabular}{>{\centering\arraybackslash}p{3cm}|
                >{\centering\arraybackslash}p{1cm}|
                >{\centering\arraybackslash}p{1cm}|
                >{\centering\arraybackslash}p{1cm}}
Dataset & $K$ & $w$ & $\eta$ \\ \hline
BabyECG  & 5   & 15  & 15  \\
UCI-HAR  & 2   & 12  & 8   \\
HoneyBee & 5   & 30  & 15  \\
USC-HAD  & 6   & 100 & 100 \\
\end{tabular}%
}
\end{table}

\begin{figure}[t]
  \centering
  \begin{tikzpicture}
    \begin{axis}[
      width=\columnwidth,
      height=0.6\columnwidth,
      title={USC-HAD},
      xlabel={Nr. queries},
      ylabel={Score},
      xmin=0, xmax=51,
      ymin=0, ymax=1,
      xtick distance=10,
      ytick distance=0.2,
      grid=both,
      grid style={line width=.1pt, draw=gray!20},
      major grid style={line width=.2pt, draw=gray!50},
      legend style={
        at={(0.97,0.03)},
        anchor=south east,
      },
      legend cell align=left,
    ]

      \addplot[name path=upper, draw=none, forget plot]
          table[
            col sep=space,
            header=true,
            x=query,
            y expr=\thisrow{mean_f1} + \thisrow{std_f1}
          ] {data/results/uschad_muralcpd.txt};
        \addplot[name path=lower, draw=none, forget plot]
          table[
            col sep=space,
            header=true,
            x=query,
            y expr=\thisrow{mean_f1} - \thisrow{std_f1}
          ] {data/results/uschad_muralcpd.txt};
        \addplot[fill=blue, fill opacity=0.2, forget plot] 
          fill between[of=upper and lower];
        \addplot[
          blue,
          very thick,
        ]
          table[
            col sep=space,
            header=true,
            x=query,
            y=mean_f1
          ] {data/results/uschad_muralcpd.txt};
        \addlegendentry{F1-score}

          \addplot[name path=upper, draw=none, forget plot]
          table[
            col sep=space,
            header=true,
            x=query,
            y expr=\thisrow{mean_prec} + \thisrow{std_prec}
          ] {data/results/uschad_muralcpd.txt};
        \addplot[name path=lower, draw=none, forget plot]
          table[
            col sep=space,
            header=true,
            x=query,
            y expr=\thisrow{mean_prec} - \thisrow{std_prec}
          ] {data/results/uschad_muralcpd.txt};
        \addplot[fill=orange, fill opacity=0.2, forget plot] 
          fill between[of=upper and lower];
        \addplot[
          orange,
          very thick,
        ]
          table[
            col sep=space,
            header=true,
            x=query,
            y=mean_prec
          ] {data/results/uschad_muralcpd.txt};
        \addlegendentry{Precision}

          \addplot[name path=upper, draw=none, forget plot]
          table[
            col sep=space,
            header=true,
            x=query,
            y expr=\thisrow{mean_rec} + \thisrow{std_rec}
          ] {data/results/uschad_muralcpd.txt};
        \addplot[name path=lower, draw=none, forget plot]
          table[
            col sep=space,
            header=true,
            x=query,
            y expr=\thisrow{mean_rec} - \thisrow{std_rec}
          ] {data/results/uschad_muralcpd.txt};
        \addplot[fill=green, fill opacity=0.2, forget plot] 
          fill between[of=upper and lower];
        \addplot[
          green,
          very thick,
        ]
          table[
            col sep=space,
            header=true,
            x=query,
            y=mean_rec
          ] {data/results/uschad_muralcpd.txt};
        \addlegendentry{Recall}

    \end{axis}
  \end{tikzpicture}
  \caption{Evolution of F1-score, precision, and recall as a function of the number of user queries on the USC-HAD dataset using MuRAL-CPD. Solid lines represent the mean values across 5 USC-HAD sequences, each repeated 10 times, while shaded regions indicate the standard deviation across these repetitions.}
  \label{fig:uschad_results}
\end{figure}
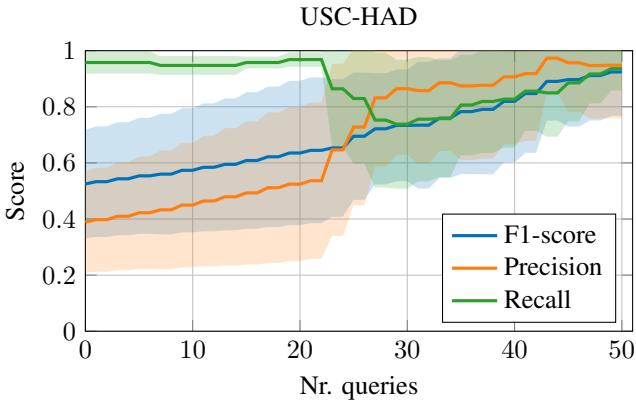
\subsection{Benchmark Datasets and Evaluation Metrics}
We consider four publicly available real-world datasets spanning three domains: human activity recognition (HAR), physiological monitoring (PM), and behavioral tracking (BT). 

In the HAR category, we include the USC-HAD~\cite{uschad} and UCI-HAR~\cite{ucihar} datasets. The USC-HAD dataset was collected using tri-axial MEMS accelerometers and includes motion recordings from 14 subjects with varying gender, age, weight, and height. Similarly, UCI-HAR contains data captured from both a tri-axial MEMS accelerometer and a gyroscope. In this case, the raw sensor signals are preprocessed to extract statistical features in both the time and frequency domains, resulting in a high-dimensional dataset with 561 channels. Due to the high redundancy among features, and in line with prior work~\cite{ICPD}, we use only the first 50 features for our experiments. In both cases change points within the series represent changes in the activity performed by the subject, spanning across various actions such as sitting and running.

In PM we consider the BabyECG~\cite{babyecg} dataset. It records the 1-dimensional heart-rate from a 66-old infant for one night. Specifically, it contains 2048 observations, one every 16 seconds and change points within it correspond to transitions between different sleep stages.

In BT we tested the Honeybee dance dataset~\cite{honeybee}. Widely used as a benchmark in CPD algorithms, it tracks the movements of a bee during a 'dance'. It consists of 6 sequences of three channels with varying length where CP corresponds to transitions between different waggle dance stages. 

In line with prior works, we adopt the evaluation criterion proposed by~\cite{TIRE}. A predicted CP $\tau$ is considered a correct detection for a ground truth change point $\tau^*$ if it lies within a tolerance window of width $\eta$ centered on $\tau^*$, and no other predicted point is closer to $\tau^*$. Formally, for a given ground truth change point $\tau^* \in \{\tau^*_j\}$, the corresponding matched prediction $\tau \in \{\tau_j\}$ must satisfy:
\begin{equation}
    |\tau - \tau^*| \leq \eta \quad \text{and} \quad \forall \hat{\tau} \in \{\tau_j\}, \ |\hat{\tau} - \tau^*| \geq |\tau - \tau^*|.
\end{equation}

Each predicted change point can match at most one ground truth point, and vice versa. Based on these one-to-one assignments, we compute standard detection metrics: precision, recall, and F1-score. Among these, we report the F1-score as the primary evaluation metric, since it is widely adopted in the CPD literature and provides a balanced measure of detection accuracy.

\subsection{Experimental Setting}\label{sec:experimental_setting}
The wavelet function used for multiscale representation via MDWD is the Daubechies-2 (db2) wavelet, and the decomposition is performed channel-wise. This wavelet defines the filters used in the convolution operations described in Equation~\ref{eq:conv}. The number of decomposition levels $K$ is chosen individually for each dataset, based on the signal length and the expected frequency components present in the data. To realign all sub-band scores with the original time axis, we apply an upsampling procedure using the \texttt{resample} function from the \texttt{scipy.signal} module, which performs Fourier-based resampling.

The key hyperparameters used in our experiments are summarized in Table~\ref{tab:dataset_parameters}, including the number of wavelet decomposition levels $K$, the sliding-window size $w$ for the Normal Discrepancy Score, and the query window tolerance $\eta$ used in the active learning module. All datasets are normalized to zero mean and unit variance before processing.

After each query iteration, the model hyperparameters $(\{\theta_k\}, \zeta)$ are re-optimized based on the $F_1$-score computed over the supervised domain $\mathcal{D}_S$, and the updated parameters are used to regenerate the score profile for the next round. The optimization is performed using the Mango optimizer~\cite{mango}, a Bayesian optimization framework which performs up to 50 function evaluations within a parameter search space of size 5000. This procedure is triggered after the first 10 user queries and subsequently repeated every two feedback iterations.

\subsection{Results}
In this section we provide a comparison between MuRAL-CPD and the actual state-of-the-art method ICPD~\cite{ICPD}. Figure~\ref{fig:results} provides the F1-scores as a function of the number of queries. Each method is evaluated across three real-world datasets, with curves averaged over 10 runs and shaded areas representing standard deviation.

We first observe in Figure~\ref{fig:results} that both MuRAL-CPD and ICPD consistently improve over their unsupervised backbone models as the number of queries increases. This confirms that incorporating user feedback through active learning leads to an improvement in segmentation performance, validating the core assumption behind both approaches.

We note that the initial performance of our method is lower than that of ICPD. This is expected: while ICPD builds upon TIRE—an encoder-decoder-based unsupervised model that produces relatively accurate initial candidate change points—MuRAL-CPD starts from a much lighter initialization, designed to be computationally efficient and highly tunable. Our method relies on a faster unsupervised algorithm with a large number of hyperparameters, which are optimized progressively throughout the active learning loop in order to actively adapt the model to the current time series. In contrast, ICPD initializes with a stronger unsupervised backbone, but its performance improves more slowly. This slower convergence is attributable to the model they rely within the active learning loop. In contrast to our approach, ICPD employs a supervised one-class SVM that classifies each sample as either a change point or a normal point, based on the initial pseudo-labels provided by TIRE and the labels progressively collected through user feedback. However, since the one-class SVM is trained exclusively on samples labeled as normal, any misclassified change points that have not yet been corrected by the user can significantly hinder the model's convergence. This dependency on clean supervision slows down the learning process, particularly in the early stages of the active loop. In contrast, our approach is unaffected by early misclassifications, since the optimization process is driven exclusively by the user-labeled change points, and does not rely on noisy pseudo-labels.

On the USC-HAD dataset, we report only the performance of MuRAL-CPD, as the preprocessing phase required by TIRE—used as a backbone in ICPD—proved computationally infeasible due to the length and complexity of the series. This is useful to concentrate on a particular behavior of our model. Notably, immediately after the first 10 queries—when the optimization procedure is started—the model adjusts the decision threshold upward to better align with the user-provided labels. This adjustment results in a sudden increase in precision, indicating a more conservative approach that prioritizes correctness over coverage. As a consequence, recall temporarily drops, reflecting the model's tendency to filter out uncertain detections.
This early rise in precision—achieved at the cost of recall—is a characteristic behavior of our approach. It shows that the model quickly learns to reduce false positives by tuning its hyperparameters based on the initial feedback. As more queries are collected, the recall gradually recovers and stabilizes.

\subsection{Ablation Study}

\begin{figure}[t]
  \centering
  \begin{tikzpicture}
    \begin{axis}[
      width=\columnwidth,
      height=0.6\columnwidth,
      title={Honeybee Dance},
      xlabel={Nr. queries},
      ylabel={F1-Score},
      xmin=0, xmax=30,
      ymin=0, ymax=1,
      xtick distance=5,
      ytick distance=0.2,
      grid=both,
      grid style={line width=.1pt, draw=gray!20},
      major grid style={line width=.2pt, draw=gray!50},
      legend style={
        at={(0.97,0.03)},
        anchor=south east,
      },
      legend cell align=left,
    ]

      \addplot[name path=upper, draw=none, forget plot]
          table[
            col sep=space,
            header=true,
            x=query,
            y expr=\thisrow{mean_f1} + \thisrow{std_f1}
          ] {data/results/honeybee_muralcpd.txt};
        \addplot[name path=lower, draw=none, forget plot]
          table[
            col sep=space,
            header=true,
            x=query,
            y expr=\thisrow{mean_f1} - \thisrow{std_f1}
          ] {data/results/honeybee_muralcpd.txt};
        \addplot[fill=blue, fill opacity=0.2, forget plot] 
          fill between[of=upper and lower];
        \addplot[
          blue,
          very thick,
        ]
          table[
            col sep=space,
            header=true,
            x=query,
            y=mean_f1
          ] {data/results/honeybee_muralcpd.txt};
        \addlegendentry{MuRAL-CPD}

          \addplot[name path=upperB, draw=none, forget plot]
          table[
            col sep=space,
            header=true,
            x=query,
            y expr=\thisrow{mean_f1} + \thisrow{std_f1}
          ] {data/ablation/ablation_threshold.txt};
        \addplot[name path=lowerB, draw=none, forget plot]
          table[
            col sep=space,
            header=true,
            x=query,
            y expr=\thisrow{mean_f1} - \thisrow{std_f1}
          ] {data/ablation/ablation_threshold.txt};
        \addplot[fill=orange, fill opacity=0.2, forget plot] 
          fill between[of=upperB and lowerB];
        \addplot[
          orange,
          very thick,
        ]
          table[
            col sep=space,
            header=true,
            x=query,
            y=mean_f1
          ] {data/ablation/ablation_threshold.txt};
        \addlegendentry{MuRAL-CPD-Max}
    \end{axis}
  \end{tikzpicture}
  \caption{Comparison between MuRAL-CPD and its ablated variant MuRAL-CPD-Max under different threshold initialization strategies on the Honeybee Dance dataset. While MuRAL-CPD benefits from a more informed threshold initialization and achieves better early performance, both methods converge to similar F1-scores after approximately 25 queries. Solid lines represent the mean across 10 runs and 6 dataset while shaded areas represent the standard deviation across those repetitions.}
  \label{fig:threshold_init}
\end{figure}
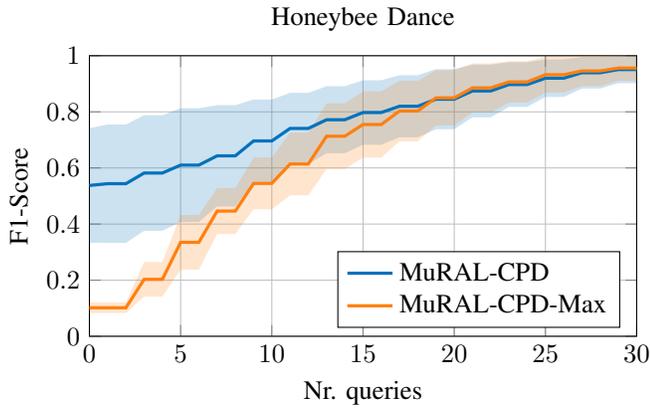

To better understand the contribution of individual components in our framework, we conduct a set of ablation studies focusing on three key design choices: the initialization of the decision threshold, the query selection policy, and the initial exploration. These experiments aim to assess the sensitivity of the model to each component and to validate the rationale behind our architectural choices.
\subsubsection{Initialization of the decision threshold}
As introduced in Section~\ref{sec:hyper_init}, the decision threshold plays a crucial role in the active learning loop. This is mostly due to the query policy adopted by MuRAL-CPD, which prioritizes samples located near the current decision boundary. As a result, the initial choice of threshold strongly influences which samples are selected for annotation in the early stages.

To evaluate this effect, we compare the standard MuRAL-CPD to a variant named MuRAL-CPD-Max, which initializes the threshold using a naive strategy: setting it equal to the maximum score value in the sequence. This choice delays the model’s ability to focus on uncertain regions, since all the most uncertain scores are far from the decision boundary.

Figure~\ref{fig:threshold_init} shows that MuRAL-CPD-Max performs noticeably worse than MuRAL-CPD in the early rounds of active learning. However, as more annotations are collected and the model parameters are refined, the performance gap gradually narrows. In the end, both methods converge to comparable levels of F1-score.

This experiment highlights two important findings: first, the initial threshold is a highly influential hyperparameter; second, despite its importance, the algorithm shows strong robustness to its initialization. Even with suboptimal initial values, the model can recover and achieve competitive segmentation performance through iterative refinement, as expected from \ref{prop:1}

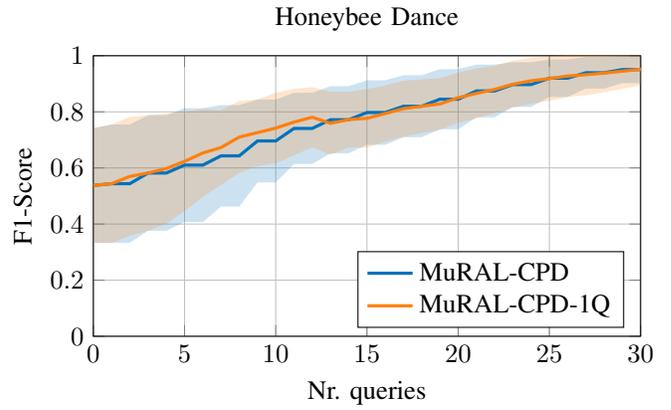
\begin{figure}[t]
  \centering
  \begin{tikzpicture}
    \begin{axis}[
      width=\columnwidth,
      height=0.6\columnwidth,
      title={Honeybee Dance},
      xlabel={Nr. queries},
      ylabel={F1-Score},
      xmin=0, xmax=30,
      ymin=0, ymax=1,
      xtick distance=5,
      ytick distance=0.2,
      grid=both,
      grid style={line width=.1pt, draw=gray!20},
      major grid style={line width=.2pt, draw=gray!50},
      legend style={
        at={(0.97,0.03)},
        anchor=south east,
      },
      legend cell align=left,
    ]

      \addplot[name path=upper, draw=none, forget plot]
          table[
            col sep=space,
            header=true,
            x=query,
            y expr=\thisrow{mean_f1} + \thisrow{std_f1}
          ] {data/results/honeybee_muralcpd.txt};
        \addplot[name path=lower, draw=none, forget plot]
          table[
            col sep=space,
            header=true,
            x=query,
            y expr=\thisrow{mean_f1} - \thisrow{std_f1}
          ] {data/results/honeybee_muralcpd.txt};
        \addplot[fill=blue, fill opacity=0.2, forget plot] 
          fill between[of=upper and lower];
        \addplot[
          blue,
          very thick,
        ]
          table[
            col sep=space,
            header=true,
            x=query,
            y=mean_f1
          ] {data/results/honeybee_muralcpd.txt};
        \addlegendentry{MuRAL-CPD}

          \addplot[name path=upperB, draw=none, forget plot]
          table[
            col sep=space,
            header=true,
            x=query,
            y expr=\thisrow{mean_f1} + \thisrow{std_f1}
          ] {data/ablation/ablation_single_lc.txt};
        \addplot[name path=lowerB, draw=none, forget plot]
          table[
            col sep=space,
            header=true,
            x=query,
            y expr=\thisrow{mean_f1} - \thisrow{std_f1}
          ] {data/ablation/ablation_single_lc.txt};
        \addplot[fill=orange, fill opacity=0.2, forget plot] 
          fill between[of=upperB and lowerB];
        \addplot[
          orange,
          very thick,
        ]
          table[
            col sep=space,
            header=true,
            x=query,
            y=mean_f1
          ] {data/ablation/ablation_single_lc.txt};
        \addlegendentry{MuRAL-CPD-1Q}
    \end{axis}
  \end{tikzpicture}
\caption{Comparison between MuRAL-CPD and its ablated variant MuRAL-CPD-1Q on the Honeybee Dance dataset. While MuRAL-CPD queries two least-certain samples per iteration before updating the model, MuRAL-CPD-1Q performs an optimization step after each individual query. Solid lines represent the mean across 10 runs on 6 different sequences, and shaded areas indicate the standard deviation.}
  \label{fig:optimize}
\end{figure}

\subsubsection{Influence of the optimization frequency}
In the MuRAL-CPD algorithm, the query policy selects two candidate samples to be labeled at each iteration. These samples are simultaneously presented to the user, after which the model parameters are re-optimized based on the newly acquired labels. In this subsection, we investigate how this simultaneous querying strategy influences the convergence speed of the algorithm.

To this end, we introduce a variant, denoted as MuRAL-CPD-1Q, in which only one sample is queried at a time and the model is re-optimized after every single annotation, rather than after every pair. 

As shown in Figure~\ref{fig:optimize}, this alternative strategy does not lead to any improvement in performance. On the contrary, it slows down the entire active learning loop. The lack of measurable benefits is likely due to the instability introduced by performing optimization on very limited supervision at each step. In contrast, optimizing after each pair of queries provides a more stable update signal and helps the model improve more consistently.

\begin{figure}[t]
  \centering
  \begin{tikzpicture}
    \begin{axis}[
      width=\columnwidth,
      height=0.6\columnwidth,
      title={Honeybee Dance},
      xlabel={Nr. queries},
      ylabel={F1-Score},
      xmin=0, xmax=30,
      ymin=0, ymax=1,
      xtick distance=5,
      ytick distance=0.2,
      grid=both,
      grid style={line width=.1pt, draw=gray!20},
      major grid style={line width=.2pt, draw=gray!50},
      legend style={
        at={(0.97,0.03)},
        anchor=south east,
      },
      legend cell align=left,
    ]

      \addplot[name path=upper, draw=none, forget plot]
          table[
            col sep=space,
            header=true,
            x=query,
            y expr=\thisrow{mean_f1} + \thisrow{std_f1},
          ] {data/results/honeybee_muralcpd.txt};
        \addplot[name path=lower, draw=none, forget plot]
          table[
            col sep=space,
            header=true,
            x=query,
            y expr=\thisrow{mean_f1} - \thisrow{std_f1}
          ] {data/results/honeybee_muralcpd.txt};
        \addplot[fill=blue, fill opacity=0.2, forget plot] 
          fill between[of=upper and lower];
        \addplot[
          blue,
          very thick,
        ]
          table[
            col sep=space,
            header=true,
            x=query,
            y=mean_f1
          ] {data/results/honeybee_muralcpd.txt};
        \addlegendentry{MuRAL-CPD}

          \addplot[name path=upperB, draw=none, forget plot]
          table[
            col sep=space,
            header=true,
            x=query,
            y expr=\thisrow{mean_f1} + \thisrow{std_f1}
          ] {data/ablation/ablation_no_warmup.txt};
        \addplot[name path=lowerB, draw=none, forget plot]
          table[
            col sep=space,
            header=true,
            x=query,
            y expr=\thisrow{mean_f1} - \thisrow{std_f1}
          ] {data/ablation/ablation_no_warmup.txt};
        \addplot[fill=orange, fill opacity=0.2, forget plot] 
          fill between[of=upperB and lowerB, forget plot];
        \addplot[
          orange,
          very thick,
        ]
          table[
            col sep=space,
            header=true,
            x=query,
            y=mean_f1
          ] {data/ablation/ablation_no_warmup.txt};
        \addlegendentry{MuRAL-CPD-No-Warm}
    \end{axis}
  \end{tikzpicture}
  \caption{Comparison between MuRAL-CPD and its ablated variant MuRAL-CPD-No-Warm on the Honeybee Dance dataset. The difference lies in the presence of an initial warm-up phase: MuRAL-CPD collects 10 labels before starting optimization, while MuRAL-CPD-No-Warm begins optimization immediately after the first query. Solid lines represent the mean across 10 runs on 6 different sequences, and shaded areas indicate the standard deviation.}
  \label{fig:warm_up}
\end{figure}
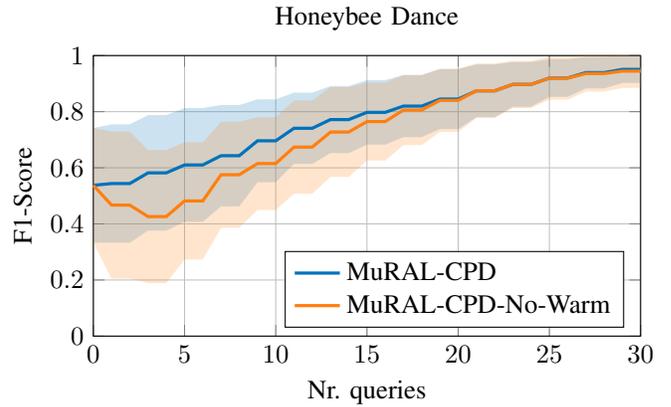

\subsubsection{Influence of the warm-up phase}
As introduced in the experimental setting described in Section~\ref{sec:experimental_setting}, the results from the MuRAL-CPD algorithm presented in the previous section share an initial warm-up phase, during which labels are collected without performing any optimization. The model parameters are updated only after the first 10 labels have been collected. In this subsection, we explore the alternative strategy of running the optimization step immediately after the first two queries. To this end, we define a fourth model, named MuRAL-CPD-No-Warm, which entirely skips the initial warm-up phase.

As illustrated in Figure~\ref{fig:warm_up}, both methods converge to the same performance level as the number of queries increases. However, one can observe a consistent initial drop in performance in MuRAL-CPD-No-Warm, occurring right after the first few queries. This degradation is mainly caused by the optimizer overfitting the very limited number of initial labels, which leads to a significant drop in the recall of the algorithm.

In conclusion, although early optimization might seem advantageous for faster feedback integration, skipping the warm-up phase proves detrimental in the initial stage. A short delay before the first update helps stabilize learning and improves the robustness of the model during early exploration.

\section{Conclusion}\label{sec:conclusions}
We presented MuRAL-CPD, a semi-supervised method for multiresolution change point detection (CPD) that integrates active learning into an unsupervised backbone. By leveraging wavelet-based multiresolution decomposition and user feedback, MuRAL-CPD achieves efficient and adaptive segmentation. A key feature of MuRAL-CPD is its robustness to the initial parameter settings, ensuring consistent final results. However, we showed that a well-chosen initial decision threshold accelerates convergence, reducing the number of user queries required. Experimental results demonstrate its competitive performance over state-of-the-art methods, particularly in scenarios with limited supervision.

As a future direction, we plan to explore reinforcement learning to optimize the query selection policy in the active learning loop. This could further minimize user effort while enhancing the efficiency and accuracy of the CPD process, making MuRAL-CPD an even more powerful tool for time series analysis.

\bibliographystyle{IEEEtran}
\bibliography{references}

\end{document}